\documentclass[letterpaper]{IEEEtran}  

\usepackage{amsmath,amssymb,amsthm}
\usepackage{listings}
\usepackage{algpseudocode}
\usepackage{graphicx}
\usepackage{subfig}
\usepackage{color}
\usepackage{hyperref}
\usepackage{tabularx}
\usepackage{booktabs,tabularx}
\usepackage{balance}
\usepackage{nth}

\newtheorem{theorem}{Theorem}

\definecolor{dkgreen}{rgb}{0,0.6,0}
\definecolor{gray}{rgb}{0.5,0.5,0.5}
\definecolor{mauve}{rgb}{0.58,0,0.82}

\newlength{\tempheight}
\newlength{\tempwidth}

\newcommand{\rowname}[1]
{\rotatebox{90}{\makebox[\tempheight][c]{\textbf{#1}}}}

\newcommand{\columnname}[1]
{\makebox[\tempwidth][c]{\textbf{#1}}}

\title{\LARGE \bf Multi-Agent Team Access Monitoring: Environments that Benefit from Target Information Sharing
}

\IEEEoverridecommandlockouts
\begin{document}

\author{%
Andrew Dudash$^{1}$, Scott James$^{2}$, and Ryan Rubel$^{3}$
\thanks{$^{1}$Noblis Autonomous Systems Research Center
                    Reston, Virginia
        {\tt\small Andrew.Dudash@noblis.org}}%
\thanks{$^{2}$Noblis Simulation and Visualization Research Center
                    Reston, Virginia
                    {\tt\small Scott.James@noblis.org}}%
\thanks{$^{3}$Noblis Autonomous Systems Research Center
                    Reston, Virginia
        {\tt\small Ryan.Rubel@noblis.org}}%
}

\maketitle

\begin{abstract}
  Robotic access monitoring of multiple target areas has applications including checkpoint enforcement, surveillance and containment of fire and flood hazards. Monitoring access for a single target region has been successfully modeled as a minimum-cut problem. We generalize this model to support multiple target areas using two approaches: iterating on individual targets and examining the collections of targets holistically. Through simulation we measure the performance of each approach on different scenarios.
\end{abstract}

\maketitle

\section{INTRODUCTION}

Consider the obstacle-filled environment in Figure \ref{fig:patho1} and Figure \ref{fig:patho2}. Three non-contiguous target regions are surrounded so that nothing can enter from the edge of the environment and reach them without being detected by a robot. Attempting to minimize the amount of robots required for each target region individually requires three robots (Figure \ref{fig:patho1}), one for each target. Conversely, by treating the target regions holistically, the amount of agents required to survey an area can be reduced to two (Figure \ref{fig:patho2}), the size of the common opening to the targets. In some scenarios, the individual and holistic approach result in identical solutions. (Figure \ref{fig:nochange1}). In this paper, we will examine the characteristics of scenarios that benefit from a holistic approach versus an individualized approach and vice versa.

We will address the following question: given an obstacle-filled environment, how many robots $n$, with limited sensor range, must be placed to protect $m$ target regions from agents entering from the edges of the environment? 
\begin{figure}
  \centering
  \subfloat[The individual approach blocks access to target regions.]{
    \label{fig:patho1}
    \includegraphics[width=35mm]{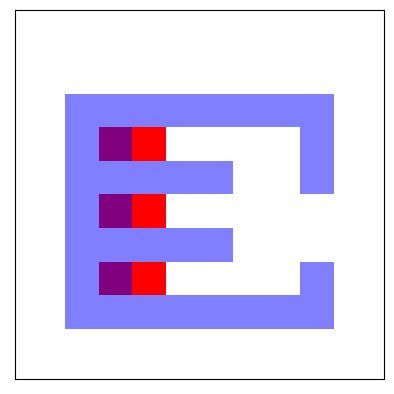}
  }\hfill
  \subfloat[The holistic approach exploits the environment, requiring one less robot than the individual approach.]{
    \label{fig:patho2}
    \includegraphics[width=35mm]{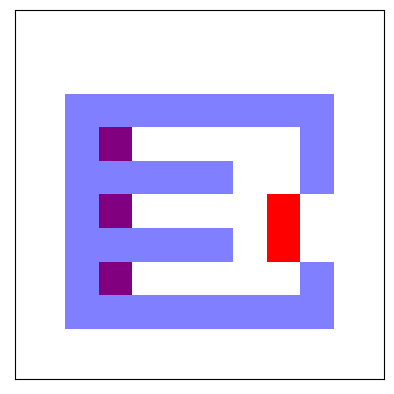}
  }

  \subfloat[The individual and holistic algorithms produce the same solution.]{
    \label{fig:nochange1}
    \includegraphics[width=35mm]{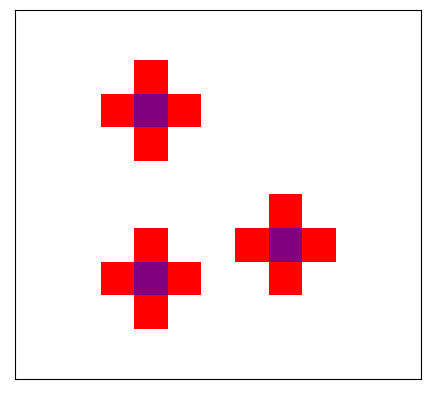}
  }\hfill
  \subfloat{
    \label{fig:legend1}
    \includegraphics[width=35mm]{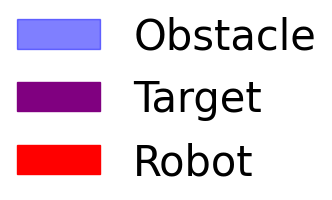}
  }

  \caption{Sharing target region information can improve solution quality for some environments.}
  \label{fig:patho}
\end{figure}

The access monitoring problem has applications in multiple areas. The counter-insurgent checkpoints described by Galula are used to temporarily isolate insecure regions until they can be combined with existing safe zones \cite{Galula}. If this were automated, an efficient solution to our access monitoring problem could ease dynamic checkpoint changes. Otte built a robotic swarm that assumes a formation when exposed to environmental conditions \cite{Otte2018AnEG} and suggests its use for handling fires, floods, tornadoes, and earthquakes; one of the response formations could be access monitoring. For surveillance, the Leschi Town combined arms collective training facility \cite{LeschiTown} was used to test the ability of drone swarms to isolate areas containing items of interest. Once the individual areas are secured, the drones could transition to monitoring all areas collectively. If collective monitoring requires less robots, as our paper explores, the remaining robots could be freed for other operations.


Gupta's solution has already shown that a method that exploits environment obstacles can reduce the number of required robots. 

In this paper, we will:

\begin{itemize}
\item Extend an existing access-monitoring technique to monitor access to multiple non-contiguous target regions (Section \ref{section:algorithm}).
\item Prove the new holistic approach provides valid solutions (Section \ref{section:algorithm-analysis}).
\item Compare the holistic approach to the individual approach in simulations for different scenarios. (Section \ref{section:experiment}).
\end{itemize}

\section{RELATED WORK} \label{section:related}

Prior research suggests an efficient access-monitoring algorithm must exploit environmental obstacles.  This is similar to how humans monitor access to a room: they only watch entryways. A solution to the access monitoring problem, presented by Gupta et al. in 2019 \cite{Gupta_Lin_Manocha_Xu_Otte_2019}, is possible by modeling the problem as a minimum-cut problem \cite{mincut21960}. This method is effective but was only implemented for single contiguous target areas. We are interested in applications that include multiple, non-contiguous target regions, for instance, multiple fires to be contained or multiple buildings to be surveilled.

Prior work in robotic surveillance solves variations of the art gallery problem: how to minimize the amount of robots required to watch a region. Katsilieris et al. develop a solution where robots with infinite sensor range secure and search an obstacle-filled environment \cite{Katsilieris_Lindhe_Dimarogonas}. Similar to our work, Kazazakis and Argyos use robots with limited range sensors, but instead of pursuing, securing, or monitoring an area, the robots are tasked to sweep an obstacle-filled environment \cite{Kazazakis_Argyros_2002}. Isler et al. recognize a class of environments where a single pursuer with a random search strategy will always locate an invader, even if the invader is arbitrarily fast and knows the position of the pursuer \cite{Isler_Kannan_Khanna_2005}. Kolling and Carpin model multi-robot surveillance as a graph problem, introducing the terms \textit{contaminate} and \textit{clear} to describe the multi-robot securing process. Our contribution is closely based on an existing solution to the \textit{isolation region surveillance problem} and the \textit{minimum robot isolation region surveillance problem} introduced by Gupta et al. \cite{Gupta_Lin_Manocha_Xu_Otte_2019}. All related work described includes simulated experiments that demonstrate the viability of one or more approaches, measure the performance of one or more approaches, or both.

Prior research in distributed robotics uses a variety of network models. In some models, communication is synchronous; latency is bounded. In others, communication is asynchronous: latency is unbounded \cite{intro_dist_prog}. Optimizers are easier to write for synchronous networks, because all information stored in the network is guaranteed to be available, but these networks are too fragile to handle network partitions \cite{10.1145/564585.564601}, latency \cite{James_Raheb_Hudak_2020}, or the failure of individual robots. Marcolino and Chaimowicz built a successful swarm avoidance algorithm, but assume peer-to-peer messages with bounded latency---no explicit time-out logic is included---and they limit their physical experiments to a laboratory environment \cite{Marcolino_Chaimowicz_2009}. Otte's distributed neural network converges, but the proof explicitly requires bounded delays \cite{Otte2018AnEG}. In contrast, Jones' foraging swarms are explicitly asynchronous and are designed to continue working, at reduced performance, when no network information is available \cite{Jones_Studley_Hauert_Winfield_2018}. A system capable of opportunistic cooperation, exploiting available information without relying on bounded latency, like Jones' foragers, is ideal. This opportunistic cooperation is difficult enough that we do not attempt to incorporate it within this paper. However, similar to how Bajcsy et al. \cite{Bajcsy_Herbert_Fridovich-Keil_Fisac_Deglurkar_Dragan_Tomlin_2018} suggest adapting their synchronous centralized system to use the asynchronous decentralized Drona system, provided by Desai, Saha, Yang et al. \cite{Drona2017}, our access monitoring work could also be made asynchronous.

\section{METHODOLOGY} \label{section:methodology}

We compare the results of our holistic and iterative algorithms on simulated environments within a discretized grid. Each square of an environment can be free, a target region, or blocked by an obstacle. Each robot can block a single free square.

Our simulator presents our map as an occupancy grid. Similar to other work in robotic surveillance, we also represent our environment as a graph where vertices correspond to  areas and edges correspond to directly traversable paths between them \cite{Kolling_Carpin_2007, Katsilieris_Lindhe_Dimarogonas}. We simulate our approaches on random environments, similar to other studies in distributed robotics \cite{Gupta_Lin_Manocha_Xu_Otte_2019, Isler_Kannan_Khanna_2005}.

Our simulation can correspond to several physical systems. In the first system, described by Gupta \cite{Gupta_Lin_Manocha_Xu_Otte_2019}, unmanned aerial vehicles aim cameras down at a 2D space. Other interpretations could include terrestrial robots in an office environment with limited reaction time or sensor range \cite{Kazazakis_Argyros_2002}. Beyond surveillance, robots could cordon off an area and alert humans in their range, blocking access to a flood zone.

\subsection{Algorithm Description} \label{section:algorithm}


There are two methods we use to calculate where to position access monitoring robots: an individual solution and a holistic solution. The individual method finds a solution for each target area individually. The total solution is the union of all individual solutions. The holistic method finds a solution for all target areas simultaneously. The individual approach extends the existing access-monitoring solution\cite{Gupta_Lin_Manocha_Xu_Otte_2019} without information sharing.


The individual approach takes Gupta's existing algorithm and applies it to each target region individually. Sink points correspond to target regions. For each target region, a planar graph is created that corresponds to the access-monitoring problem for that single region. One of the advantages of the individual approach is that it is easily parallelizable. We account for this advantage when we compare calculation times in Section \ref{section:results}.

\begin{lstlisting}[caption={Individual Algorithm}]
positions = set()
for target in targets:
  new_positions = minimum_cut(source, target)
  positions = union(positions, new_positions)
return positions
\end{lstlisting}

The holistic approach considers all targets simultaneously. A single possibly non-planar graph is created that corresponds to the holistic access-monitoring problem. The graph may be non-planar because non-contiguous target regions are joined together. Figure \ref{fig:dual} shows the target regions are adjacent to a common target region node.

For multiple non-contiguous targets, the graph generated by the holistic approach is almost always non-planar. If, for example, there is only one target, then the graph will be planar, but the problem is a single target problem. If all targets form a contiguous region, then the graph will be planar, but the problem is now a continguous target problem. It is, however, possible to construct a multiple non-contiguous target problem where the holistic approach generates a planar graph.

\begin{lstlisting}[caption={Holistic Algorithm}]
add_node(sink)
for target in targets:
  for neighbor in neighbors(target):
    add_edge(neighbor, sink)
    remove_edge(neighbor, target)
positions = min_cut(source, sink)
return positions
\end{lstlisting}

Both approaches use the same minimum-cut algorithm, a preflow-push algorithm\cite{Hagberg_Schult_Swart_2008}, to place robots. In contrast to Gupta, we use a minimum-node cut rather than a minimum-edge cut; to partition our graphs, we remove nodes instead of edges. Once a traversability graph is defined, the preflow-push algorithm finds a minimum-cut partitioning the target regions from the border region. The cut nodes correspond to the robot positions; all paths from the border region to a target region must pass through a robot position.


The main difference between the two approaches is how they combine information. In the holistic approach, \mbox{information} is combined early; all target areas are joined at the beginning; there is one traversability graph. In the individual approach, information is combined late; robot positions are combined after they consider each individual target; there is a different traversability graph for each target.

The collection of robot positions is a set so the same grid cell is never watched by more than one robot.  Robots are not double counted for either approach.

\begin{figure}
  \setlength{\tempwidth}{35mm}
  \setlength{\tempheight}{35mm}
  \centering
    \subfloat[]{
    \label{fig:dual1}
    \includegraphics[width=40mm]{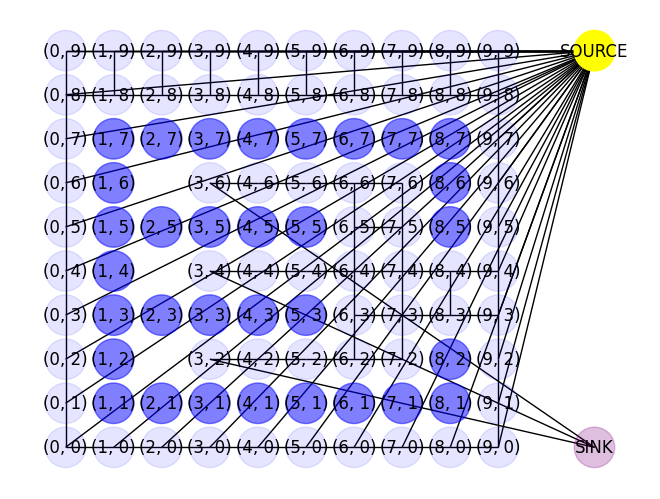}
  }\hfill
  \subfloat[]{
    \label{fig:dual2}
    \includegraphics[width=35mm]{helpful-with_sink_merge.png}
  }
  \caption{Each node in \ref{fig:dual1} corresponds to a discretized space of \ref{fig:dual2}. All border regions are contracted to a common node. This is the holistic approach and the graph is non-planar; the three target nodes are all adjacent to one common target node.}
  \label{fig:dual}
\end{figure}

\subsection{Algorithm Analysis} \label{section:algorithm-analysis}

The holistic algorithm may provide better solutions, but it only works if partitioning the combined target regions and source has the same effect as partitioning each target region from the source individually. We prove this below.

For our analysis, we reuse the nomenclature of Gupta \cite{Gupta_Lin_Manocha_Xu_Otte_2019}, summarized in Table \ref{tab:nomenclature1}, with an alteration: we have multiple areas of interest and each target region is a single node.

\begin{table}[h!]
  \centering
    \caption{The different variables in our model.}
    \label{tab:nomenclature1}
    \begin{tabular}{l|l}
      $\textbf{Symbol}$ & $\textbf{Meaning}$ \\
      \hline
      $X_{obs}$ & The regions blocked by obstacles. \\
      $X_{int}^{i}$ & Target region of interest $i$. \\
      $X_{free}$ & The free region. \\
      $\delta X$ & The border region. \\
      $B_i$ & The regions surveilled by machines. \\      
      $\delta X_{iso}$ & The combined region of all blocked and surveilled regions. \\
    \end{tabular}
\end{table}
 
\begin{theorem}
  Let there be a graph dual $G$ of the given environment.
  Let there be $n$ target regions $X_{int}^{i}$ and let $\delta X$ be the border region subgraph.
  Let there be an extra node $S$ that all $X_{int}^{i}$ are adjacent to.
  Let there be a partition separating $\delta X$ from $X_{int}^{}$.
  If the partition separates $\delta X$ from $X_{int}^{}$, it will still separate $\delta X$ from $X_{int}^{}$ with $S$ cut out.
\end{theorem}
 
\begin{proof}
  \begin{enumerate}
    \item Suppose that the graph $G$ is cut to partition the $X_{int}$ and $\delta X$ nodes, then the graph has been split into two graphs: $A$, containing $X_{int}$, and $B$, containing $\delta X$; $A$ and $B$ are disconnected.
    \item Disconnection is monotone; it applies to every subgraph. If $A$ is disconnected from $B$, then all subgraphs of $A$ are disconnected from all subgraphs of $B$.
    \item Therefore, no $X_{int}^{i}$, a subgraph of $A$, is connected to any subgraph of $B$, including $\delta X$.
    \item Removing a node or edges cannot introduce a new path.
  \end{enumerate}
  With or without the $S$ node, there is no path from $\delta X$ to any $X_{int}^{i}$. If the partition separates $\delta X$ from $X_{int}^{}$, it will still separate $\delta X$ from $X_{int}^{}$ with $S$ cut out.
\end{proof}

\subsection{Simulator Design}


We will simulate random environments and measure which environments are preferred by the two approaches.

Rather than use continuous space, like in \cite{Gupta_Lin_Manocha_Xu_Otte_2019}, we restrict our simulation to a discrete grid. The width of a grid space is the range---sensor, alarm distance, block range---of a robot.

Our model could be extended to continuous space, but based on prior results, this seems unnecessary. Gupta tested an access monitoring approach on continuous space and used two different ways to discretize continuous space, a predefined lattice and Delaunay triangulation. Despite a gap in performance between the two methods, they followed each other closely\cite{Gupta_Lin_Manocha_Xu_Otte_2019}. For this reason, we doubt that a mapping to continuous space would impact our experiments and instead use a discrete grid.

We will experiment with two types of environments: an open environment with random obstacles and a closed environment where a grid of intersections are randomly blocked. The open environment might correspond to an outdoor area dotted with large irregular obstacles. The closed environment might correspond to an urban environment with regular obstacles, like roads in a dense city.

In the simulation, we aim to protect the target areas of interest. We assume any unobstructed border can be a source of contaminants. In both experiments, a graph of the traversable areas is generated \cite{Hagberg_Schult_Swart_2008} to represent the minimum robot isolation problem. The program builds a random environment. The program is configured by environment size, obstacle count, target region count, and random number generator seed. After environment obstacles are generated, target regions are randomly placed in unobstructed areas. For each environment, the problem is solved using both approaches: holistic and individual. The parameters of each test, including a seed value, are saved with the results.

In the open environment, the map is initially empty, then random obstacles are generated. In the closed environment, the map is initially set to a grid, then random intersections are blocked.


The output of the simulation includes the calculation time for the solution and the list of positions to place robots.

Note that the obstacles in closed environments are smaller than the obstacles generated in open environments. For this reason, our experiments with closed environments, in Section \ref{section:experiment}, use a higher obstacle count.

\section{EXPERIMENTS} \label{section:experiment}


We design five simulator experiments. In the first experiment, we compare our solutions on a single hand-picked environment. In the second experiment, we generate open environments and vary the number of random obstacles. In the third experiment, we generate closed environments and vary the number of randomly blocked intersections. In the fourth experiment, we generate open environments and vary the number of target areas. In the fifth experiment, we generate a closed environment and vary the size of the environment.

\subsection*{Experiment 1 - Pathologic}

Our environment corresponds to Figure \ref{fig:patho}, generated by our simulator.

\subsection*{Experiment 2 - Open-Environment Obstacle Sweep}

In the second experiment, the performance of the two methods on an open environment with varying obstacles is determined. We create 1000 random environments for 16 different obstacle counts: 10 to 235 by increments of 15. This range of obstacle counts goes from a near empty map, similar to an open or empty parking lot, to a near completely filled map, similar to a cave system or debris covered area. Each environment has a height and width of 100 units. For each environment, there is a random number of targets between 15-20. The targets are randomly placed. There are 16000 environments in total. Both approaches are run on every environment.

\subsection*{Experiment 3 - Closed-Environment Obstacle Sweep}

In the third experiment, the performance of the two methods on a closed environment with varying obstacles is determined. We create 1000 random environments for 31 different obstacle counts: 10 to 1510 by increments of 50. This range of obstacle counts goes from a near empty map, similar to an open city grid, to a near completely filled map, similar to a city with closed or obstructed streets. Each environment has a height and width of 100 units. For each environment, there is a random number of targets between 15-20. The targets are randomly placed. There are 31000 environments in total. Both approaches are run on every environment.

The obstacle count is higher in Experiment 3 than Experiment 2 because the obstacles in Experiment 3 are smaller than the obstacles in Experiment 2.

\subsection*{Experiment 4 - Open-Environment Target Area Count Sweep}

In the fourth experiment, we measure how the performance of the holistic method changes as the number of target areas varies. We create 100 random environment for 12 different target area counts: 1 to 551 by increments of 50. The targets are randomly placed. Each environment has a height and width of 100 units. For each environment, there are 100 obstacles---an obstacle count that the holistic method was discovered, in Experiment 3, to work well on. There are 1200 environments in total. Both approaches are run on every environment.

\subsection*{Experiment 5 - Closed-Environment Environment Size Sweep}

In the fifth experiment, we measure how the performance of the holistic method changes as the size of environments varies. We create 1000 random environments for 9 different environment widths: 50 to 370 by increments of 40. The environments are square; the width and height are the same. For each environment, there is a random number of targets between 15-20. The targets are randomly placed. The block count is 950 intersections. There are 9000 environments in total. Both approaches are run on every environment.

\section{RESULTS}
\label{section:results}


To compare the performance of different access monitoring approaches, we generate solutions with each method and record the amount of robots required and time taken to calculate the solutions.

\subsection*{Experiment 1 Results - Pathologic}

The individual solution and our solution, as expected, generate solutions that correspond to Figure \ref{fig:patho}. The holistic solution protects all target regions but requires one less robot than the individual solution.

\subsection*{Experiment 2 Results - Open-Environment Obstacle Sweep}

\begin{figure}
  \centering    
  \subfloat{\includegraphics[keepaspectratio,width=0.45\linewidth]{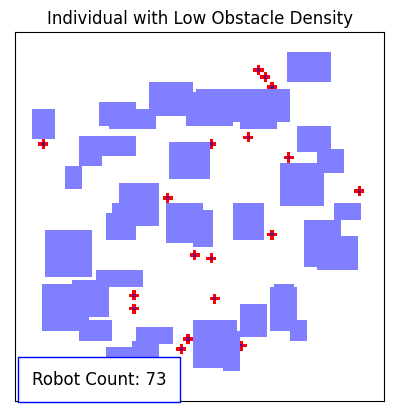}}
  \hspace{0.01em}
  \subfloat{\includegraphics[keepaspectratio,width=0.45\linewidth]{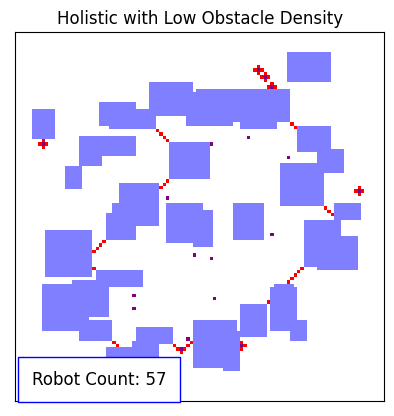}}

  \subfloat{\includegraphics[keepaspectratio,width=0.45\linewidth]{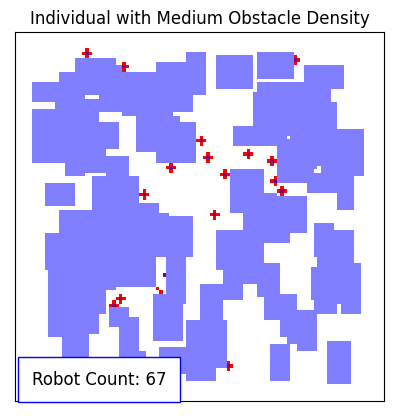}}
  \hspace{0.01em}
  \subfloat{\includegraphics[keepaspectratio,width=0.45\linewidth]{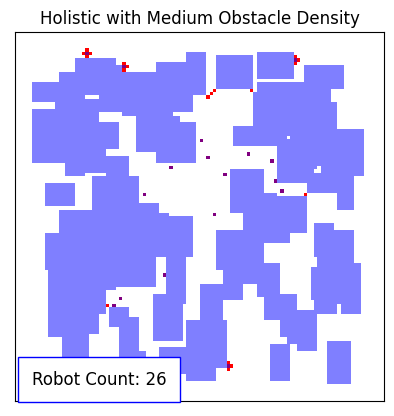}}

  \subfloat{\includegraphics[keepaspectratio,width=0.45\linewidth]{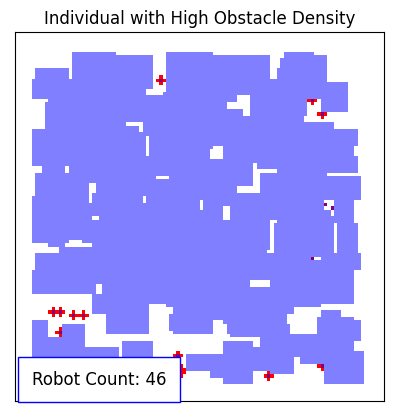}}
  \hspace{0.01em}
  \subfloat{\includegraphics[keepaspectratio,width=0.45\linewidth]{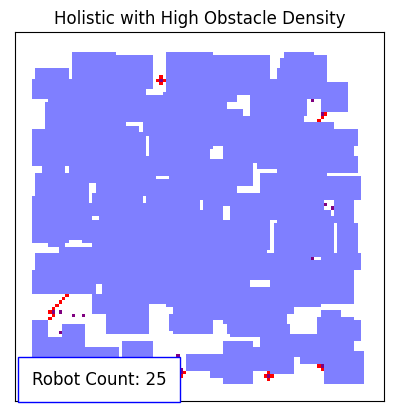}}

  \subfloat{\includegraphics[width=70mm]{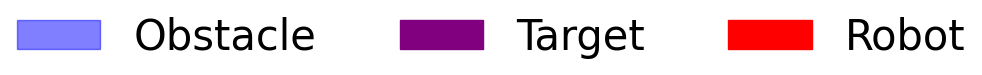}}

  \caption{Experiment 2 compared the performance of the individual and holistic approach on open environments with varied obstacle counts.}
  \label{figure:rectangle-obstacle-sweep}
\end{figure}

In Experiment 2, for each algorithm we compare the calculation time and number of robots required against the number of obstacles generated. We include sample environments of our sweep in Figure \ref{figure:rectangle-obstacle-sweep}.

\begin{figure}
  \centering
  \includegraphics[width=85mm]{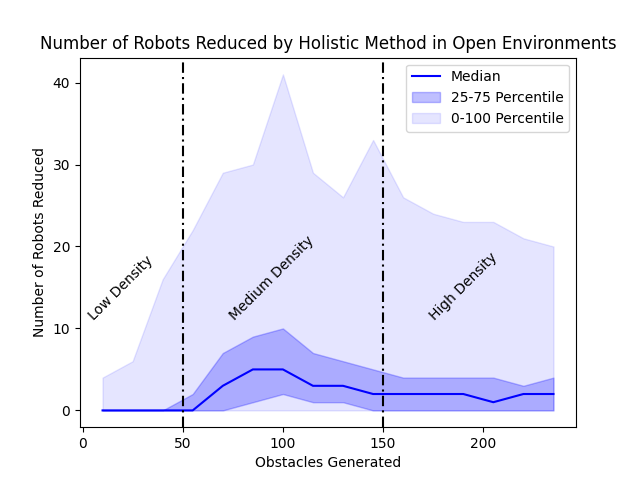}
  \caption{For all open environments tested, the holistic approach allocates fewer robots than the individual approach.  Medium density environments show the greatest improvement.}
  \label{fig:sweep-open1} 
\end{figure}

In Figure \ref{fig:sweep-open1}, we plot the median and percentiles of performance. For all trials in Experiment 2, the holistic method performs as good as or better than the individual method with regards to the number of machines required.

We measured performance as obstacle density increased. When there are few obstacles, the two methods perform similarly. Because the environments are almost completely empty, there are no environment obstacles for the holistic approach to exploit. As the number of obstacles increases, the holistic method shows significant improvement over the iterative approach. This effect weakens as the amount of obstacles increases and the performance of the two methods slowly converges. We attribute the convergence to overcrowding. As the map fills up, obstacles isolate robots from each other and the holistic method's information sharing becomes less useful. We label these three sections in Figure \ref{fig:sweep-open1}. The holistic algorithm shows greatest improvement when the environment has a medium obstacle density.

For all obstacle densities, including the holistic method's preferred medium obstacle density, there are environments where the holistic and individual solutions require the same number of machines. In Figure \ref{fig:sweep-open1}, we plot the median and percentile, rather than the mean and standard deviation, to better show the outliers; the performance of the holistic method is sensitive to the specific environment.

\begin{figure}
  \centering
  \includegraphics[width=85mm]{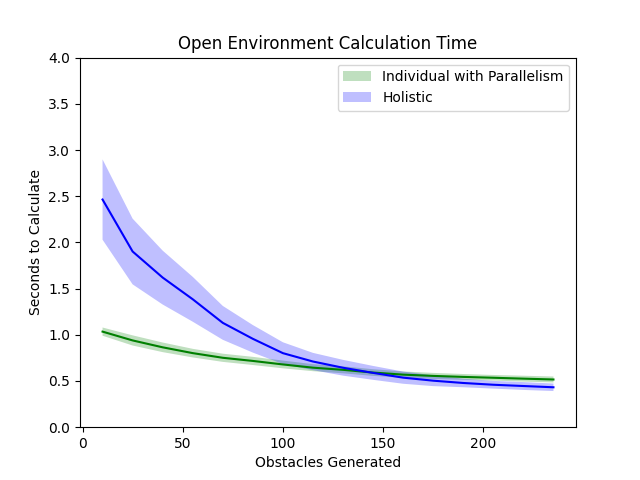}
    \caption{In open environments with high obstacle count, the holistic approach shows a computational advantage compared to the parallelized, individual approach.  Conversely, the iterative approach shows a computational advantage for low obstacle count.}
  \label{fig:sweep-open2}
\end{figure}

In open environments, the holistic method calculates an order of magnitude faster than the individual method, because the individual method must run the preflow-push algorithm once for each target region. However, this does not consider the ability of the individual method to be parallelized: all runs of the preflow-push algorithm can be run at the same time. To estimate the speed of the parallelized iterative algorithm, we divide the individual method calculation time by the number of target regions.


We compare the calculation time of the different approaches in Figure \ref{fig:sweep-open2}. The holistic method calculates slowest when the map is sparse. As the obstacle count increases, however, the holistic method eventually outperforms the parallelized individual method.

\subsection*{Experiment 3 Results - Closed-Environment Obstacle Sweep}

\begin{figure}
  \centering    
  \subfloat{\includegraphics[keepaspectratio,width=0.45\linewidth]{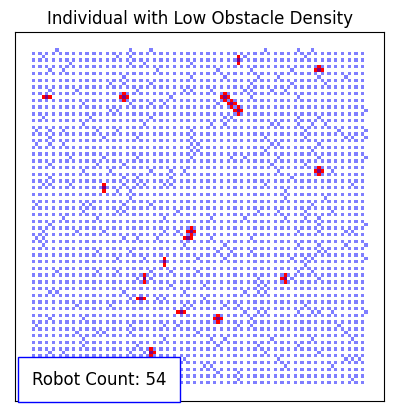}}
  \hspace{0.01em}
  \subfloat{\includegraphics[keepaspectratio,width=0.45\linewidth]{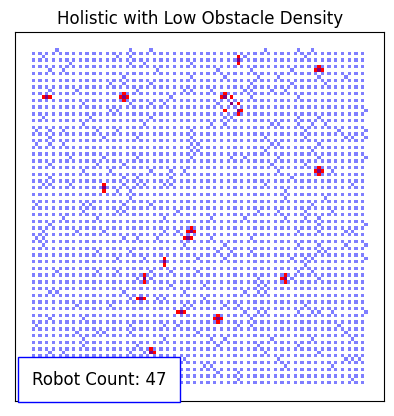}}

  \subfloat{\includegraphics[keepaspectratio,width=0.45\linewidth]{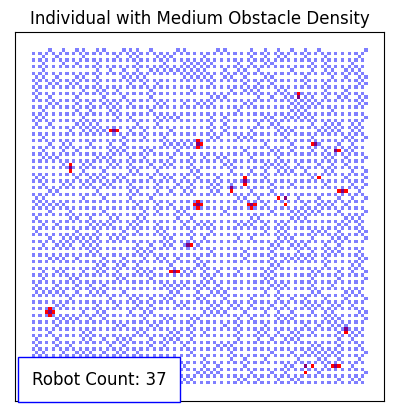}}
  \hspace{0.01em}
  \subfloat{\includegraphics[keepaspectratio,width=0.45\linewidth]{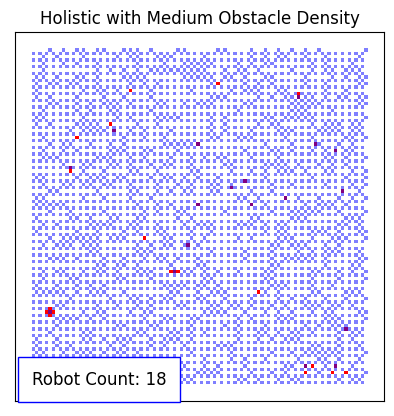}}

  \subfloat{\includegraphics[keepaspectratio,width=0.45\linewidth]{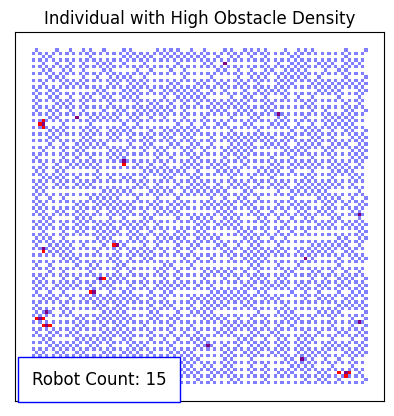}}
  \hspace{0.01em}
  \subfloat{\includegraphics[keepaspectratio,width=0.45\linewidth]{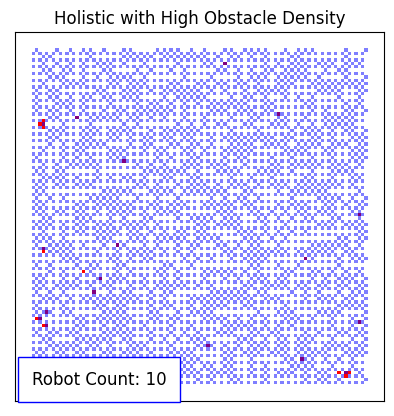}}

  \subfloat{\includegraphics[width=70mm]{legend2.png}}

  \caption{Experiment 3 compared the performance of the individual and holistic approach on closed environments with varied obstacle counts.}
  \label{figure:grid-obstacle-sweep}
\end{figure}

We contrast the performance of the two approaches on closed environments. We plot our results as two graphs. In the first graph, Figure \ref{fig:sweep-closed1}, we plot the difference in robot counts between the two approaches. In the second graph, Figure \ref{fig:sweep-closed2}, we compare the calculation time of each approach. Figure \ref{figure:grid-obstacle-sweep} shows sample environments from the experiment.

\begin{figure}
  \centering
  \includegraphics[width=85mm]{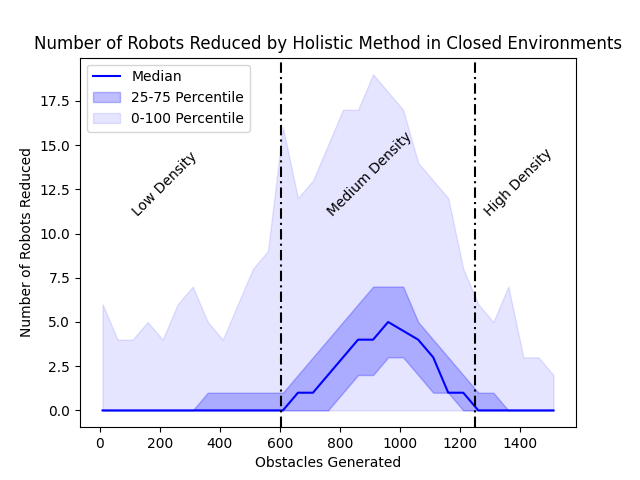}
  \caption{For all closed environments tested, the holistic approach allocates fewer robots than the individual approach.  Medium density environments show the greatest improvement.}
  \label{fig:sweep-closed1} 
\end{figure}

Similar to the results of Experiment 2, the holistic algorithm performs better when the environment is slightly crowded, but not when the environment is sparse or extremely crowded. This can be seen in Figure \ref{fig:sweep-closed1}.

The relation between machine count and obstacle density in Experiment 3 is more pronounced than in Experiment 2. We suspect this is because the obstacles in open environments, Experiment 2, can overlap, but the obstacles in closed environments, Experiment 3, never overlap. Each obstacle added to an open environment is increasingly likely to block space that an existing obstacle already blocked. In contrast, the amount of obstacles in Experiment 3 is proportional to the amount of space blocked.

The calculation time results are similar to the results of Experiment 2.  When the environment is sparse, the parallelized individual approach is faster; when the environment is dense, the holistic approach outperforms the individual approach. See Figure \ref{fig:sweep-closed2}.

\begin{figure}
  \centering
  \includegraphics[width=85mm]{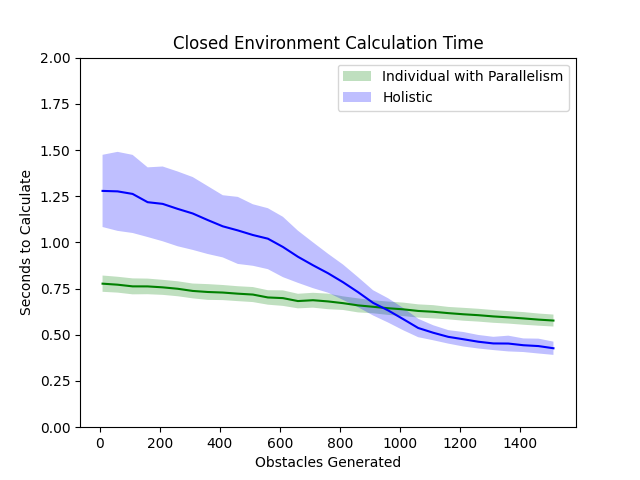}
  \caption{In closed environments with high obstacle count, the holistic approach shows a computational advantage compared to the parallelized, individual approach.  Conversely, the iterative approach shows a computational advantage for low obstacle count.}
  \label{fig:sweep-closed2}
\end{figure}

\subsection*{Experiment 4 - Open-Environment Target Area Count Sweep}

We contrast the performance of two approaches on a moderately crowded environment, 100 obstacles, as the target area count changes. The results in Figure \ref{fig:sweep-target-open1} show that the holistic method saves more machines as the number of target areas increases. As an environment fills with randomly placed machines, surrounding the entirety of the space becomes more efficient. The individual samples, shown in Figure \ref{fig:rectangle-target-sweep}, demonstrate this effect.

The figure shows a direct improvement, but we doubt that the performance can increase without bound. We anticipate a saturation point where either method, holistic or iterative, must block the border of the environment to monitor all machines. However, in our current results the holistic method performs better as the number of target areas increases.

\begin{figure}
  \centering
  \includegraphics[width=85mm]{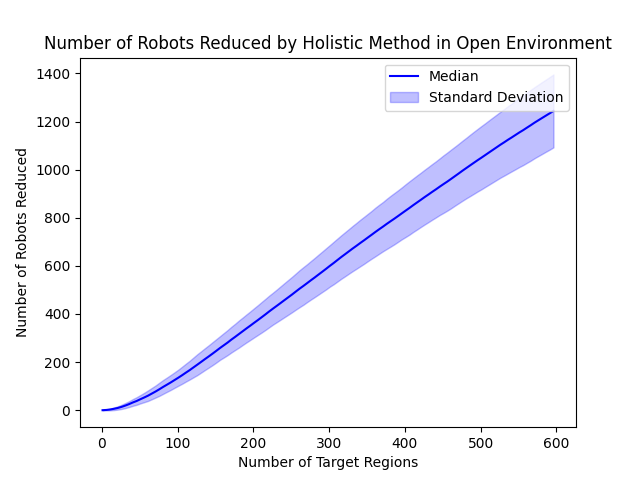}
  \caption{As the target count increases, the holistic algorithm performs better.}
  \label{fig:sweep-target-open1}
\end{figure}

\begin{figure}
  \centering
  \subfloat{\includegraphics[keepaspectratio,width=0.45\linewidth]{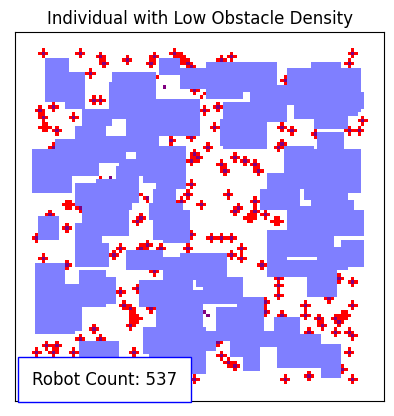}}
  \hspace{0.01em}
  \subfloat{\includegraphics[keepaspectratio,width=0.45\linewidth]{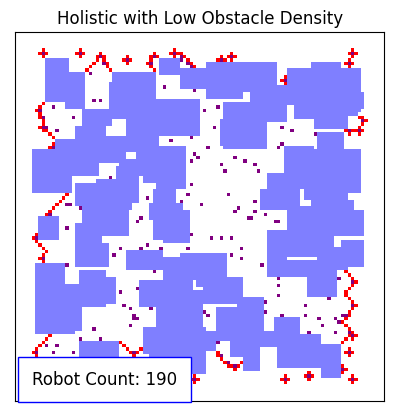}}

  \subfloat{\includegraphics[keepaspectratio,width=0.45\linewidth]{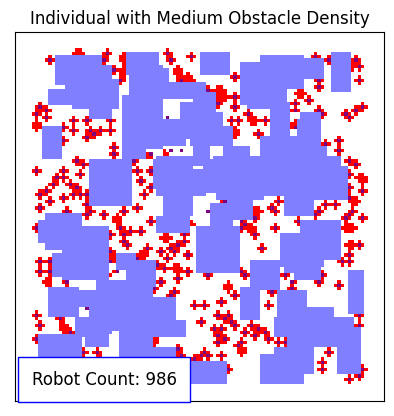}}
  \hspace{0.01em}
  \subfloat{\includegraphics[keepaspectratio,width=0.45\linewidth]{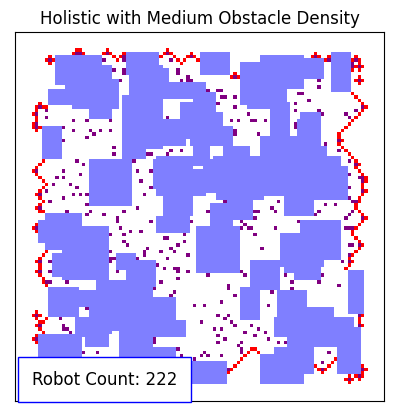}}

  \subfloat{\includegraphics[keepaspectratio,width=0.45\linewidth]{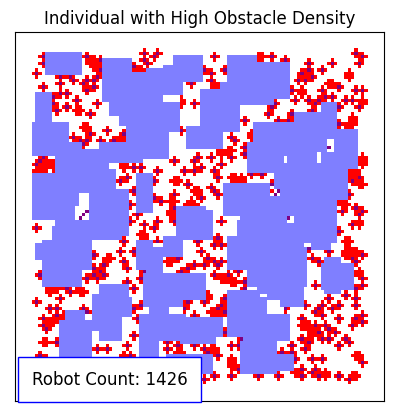}}
  \hspace{0.01em}
  \subfloat{\includegraphics[keepaspectratio,width=0.45\linewidth]{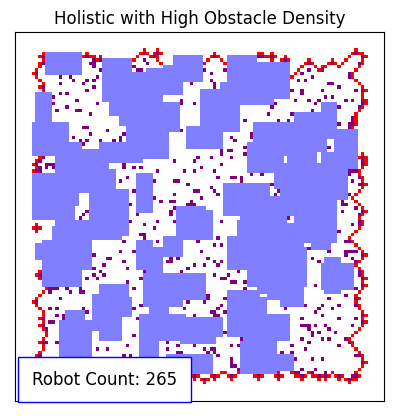}}

  \subfloat{\includegraphics[width=70mm]{legend2.png}}

  \caption{As the number of targets increases, the robot placements suggested by the holistic approach begin to outline the border of the environment.}
  \label{fig:rectangle-target-sweep}
\end{figure}

\subsection*{Experiment 5 - Closed-Environment Environment Size Sweep}

We contrast the performance of two approaches on a grid environment with varying size. The results in Figure \ref{fig:sweep-grid-size1} show that the holistic method saves more machines while the environment is small, but rapidly loses its effectiveness as the environment size increases. With a fixed obstacle and target area count but larger environment size, the environment becomes very sparse. This is similar to an environment with a low obstacle count. The holistic algorithm shows little improvement as the environment size increases.

\begin{figure}
  \centering
  \includegraphics[width=85mm]{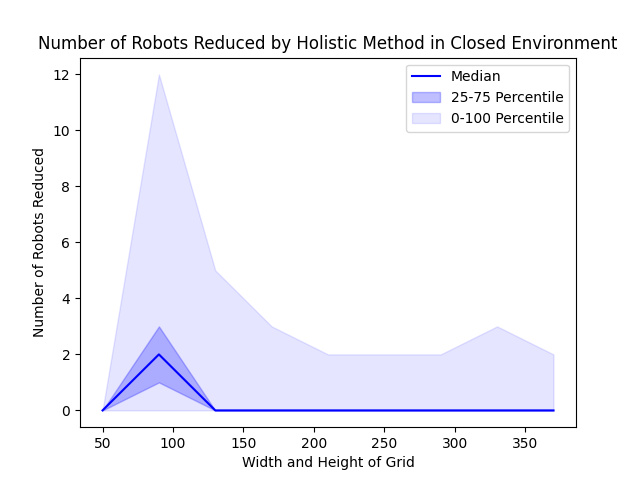}
  \caption{As the environment size increases, the holistic algorithm struggles to exploit the environment.}
  \label{fig:sweep-grid-size1}
\end{figure}

\section{CONCLUSION}

We studied if and when monitoring access to a group of target areas can be more efficient than monitoring access to each target area individually. We examined environments with irregularly located obstacles and environments where obstacles were arranged in a regular grid.

Our results revealed that in medium density environments, robots can monitor access to a group of target areas more efficiently than monitoring access to each target area individually. The holistic approach showed fewer improvements in sparse environments and in environments with a high density of obstacles. In addition, increasing the size of the environment decreased the effectiveness of the holistic approach but increasing the number of target areas increased the effectiveness. Finally, we proved that the holistic algorithm provides valid solutions.



A more thorough study of different environment types, based on the properties of the corresponding traversability graph, could reveal when environments benefit from information sharing.

\balance

\bibliography{sources}

\begin{thebibliography}{10}

\bibitem{Galula}
D.~Galula, {\em Counterinsurgency Warfare: Theory and Practice}.
\newblock Praeger Security International, 2006.

\bibitem{Otte2018AnEG}
M.~W. Otte, ``An emergent group mind across a swarm of robots: Collective
  cognition and distributed sensing via a shared wireless neural network,''
  {\em The International Journal of Robotics Research}, vol.~37, pp.~1017 --
  1061, 2018.

\bibitem{LeschiTown}
SwarmTex, ``Teams demonstrate swarm tactics in fourth major offset field
  experiment,'' 2018.

\bibitem{Gupta_Lin_Manocha_Xu_Otte_2019}
M.~Gupta, M.~C. Lin, D.~Manocha, H.~Xu, and M.~Otte, ``Monitoring access to
  user defined areas with multi-agent team in urban environments,'' in {\em
  2019 International Symposium on Multi-Robot and Multi-Agent Systems (MRS)},
  (New Brunswick, NJ, USA), pp.~56--62, IEEE, Aug 2019.

\bibitem{mincut21960}
O.~Goldschmidt and D.~Hochbaum, ``Polynomial algorithm for the k-cut problem,''
  in {\em [Proceedings 1988] 29th Annual Symposium on Foundations of Computer
  Science}, pp.~444--451, 1988.

\bibitem{Katsilieris_Lindhe_Dimarogonas}
F.~Katsilieris, M.~Lindhe, and D.~V. Dimarogonas, ``Demonstration of
  multi-robot search and secure,'' (Anchorage, AK, USA), IEEE, 2010.

\bibitem{Kazazakis_Argyros_2002}
G.~Kazazakis and A.~Argyros, ``Fast positioning of limited-visibility guards
  for the inspection of {2D} workspaces,'' in {\em IEEE/RSJ International
  Conference on Intelligent Robots and System}, vol.~3, (Lausanne,
  Switzerland), p.~2843–2848, IEEE, 2002.

\bibitem{Isler_Kannan_Khanna_2005}
V.~Isler, S.~Kannan, and S.~Khanna, ``Randomized pursuit-evasion in a polygonal
  environment,'' {\em IEEE Transactions on Robotics}, vol.~21, p.~875–884,
  Oct 2005.

\bibitem{intro_dist_prog}
C.~Cachin, R.~Guerraoui, and L.~Rodrigues, ``Introduction to reliable and
  secure distributed programming,'' pp.~44--47, Springer Publishing Company,
  Incorporated, 2011.

\bibitem{10.1145/564585.564601}
S.~Gilbert and N.~Lynch, ``Brewer's conjecture and the feasibility of
  consistent, available, partition-tolerant web services,'' {\em SIGACT News},
  vol.~33, p.~51–59, jun 2002.

\bibitem{James_Raheb_Hudak_2020}
S.~James, R.~Raheb, and A.~Hudak, ``{UAV} swarm path planning,'' in {\em 2020
  Integrated Communications Navigation and Surveillance Conference (ICNS)},
  (Herndon, VA, USA), pp.~2G3--1--2G3--12, IEEE, Sep 2020.

\bibitem{Marcolino_Chaimowicz_2009}
L.~S. Marcolino and L.~Chaimowicz, ``Traffic control for a swarm of robots:
  Avoiding group conflicts,'' in {\em 2009 IEEE/RSJ International Conference on
  Intelligent Robots and Systems}, (St. Louis, MO, USA), pp.~1949--1954, IEEE,
  Oct 2009.

\bibitem{Jones_Studley_Hauert_Winfield_2018}
S.~Jones, M.~Studley, S.~Hauert, and A.~Winfield, {\em Evolving Behaviour Trees
  for Swarm Robotics}, vol.~6 of {\em Springer Proceedings in Advanced
  Robotics}, pp.~487--501.
\newblock Cham: Springer International Publishing, 2018.

\bibitem{Bajcsy_Herbert_Fridovich-Keil_Fisac_Deglurkar_Dragan_Tomlin_2018}
A.~Bajcsy, S.~Herbert, D.~Fridovich-Keil, J.~Fisac, S.~Deglurkar, A.~Dragan,
  and C.~Tomlin, ``A scalable framework for real-time multi-robot, multi-human
  collision avoidance,'' in {\em 2019 International Conference on Robotics and
  Automation, ICRA 2019}, Proceedings - IEEE International Conference on
  Robotics and Automation, (United States), pp.~936--943, Institute of
  Electrical and Electronics Engineers Inc., May 2019.

\bibitem{Drona2017}
A.~Desai, I.~Saha, J.~Yang, S.~Qadeer, and S.~Seshia, ``Drona: A framework for
  safe distributed mobile robotics,'' in {\em Proceedings of the 8th
  International Conference on Cyber-Physical Systems}, ACM, 2017.

\bibitem{Kolling_Carpin_2007}
A.~Kolling and S.~Carpin, ``The graph-clear problem: definition, theoretical
  properties and its connections to multirobot aided surveillance,'' in {\em
  2007 IEEE/RSJ International Conference on Intelligent Robots and Systems},
  (San Diego, CA, USA), p.~1003–1008, IEEE, Oct 2007.

\bibitem{Hagberg_Schult_Swart_2008}
A.~A. Hagberg, D.~A. Schult, and P.~J. Swart, ``Exploring network structure,
  dynamics, and function using networkx,'' in {\em Proceedings of the 7th
  Python in Science Conference} (G.~Varoquaux, T.~Vaught, and J.~Millman,
  eds.), (Pasadena, CA USA), pp.~11--15, 2008.

\end{thebibliography}
\bibliographystyle{ieeetr}

\end{document}